\documentclass[]{fundam}
\usepackage{url} 
\usepackage[ruled,lined]{algorithm2e}
\usepackage{graphicx}
\usepackage{tikz}
\usepackage{algpseudocode}
\usepackage{amsmath}
\usepackage{amssymb}
\usetikzlibrary{automata, positioning, arrows}

\begin{document}

%
\setcounter{page}{1}
\publyear{2021}
\papernumber{0001}
\volume{178}
\issue{1}
%

\title{Attribute reduction and rule acquisition of formal decision context based on two new kinds of decision rules}

\address{School of Computing and Artificial Intelligence, Southwest Jiaotong University, Chengdu, Sichuan, 610031, China}

\author{Qian Hu\thanks{This work is supported by the Natural Science Foundation (grants no: 61976130)}\\
School of Computing and Artificial Intelligence\\
Southwest Jiaotong University \\ Chengdu, Sichuan, 610031, China\\
15227193972@163.com
\and Keyun Qin\\
School of Mathematics \\
Southwest Jiaotong University \\ Chengdu, Sichuan, 610031, China\\}

\maketitle

\runninghead{Qian Hu, Keyun Qin}{Attribute reduction and rule acquisition based on two new kinds of decision rules}

\begin{abstract}
  This paper mainly studies the rule acquisition and attribute reduction for formal decision context based on
two new kinds of decision rules, namely I-decision rules and II-decision rules.
The premises of these rules are object-oriented concepts, and the conclusions are formal concept and property-oriented concept respectively.
The rule acquisition algorithms for I-decision rules and II-decision rules are presented.
Some comparative analysis of these algorithms with the existing algorithms are examined which shows that the algorithms presented in this study
behave well.
The attribute reduction approaches to preserve I-decision rules and II-decision rules are presented by using discernibility matrix.
\end{abstract}

\begin{keywords}
Formal concept analysis, object-oriented and property-oriented concept lattice, rule acquisition, attribute reduction
\end{keywords}

\section{Introduction}

German scholar Wille put forward Formal concept analysis (FCA) in 1982 \cite{1}. FCA is a mathematical theory for qualitative analysis
of relation data between object and attribute that uses a formal context as input to identify a set of formal concepts formed in a concept lattice.
A formal context is a binary relation between objects set and attributes set to specify which object possess what attribute. A formal concept consists of two parts (an extent and an intent) by two derivation operators. The extent of a formal concept is an objects subset that are instances of the concept, while the intent is the subset of attributes possessed by the objects. Therefore, formal concepts are the mathematization of philosophical concepts.
As a practical tool for knowledge discovery, FCA has been successfully used in several areas, for instance data mining, information retrieval, social network analysis and machine learning \cite{2,3,4,5,6}.
In addition, some natural generalizations of derivation operators were proposed which induces some notions, for example, object-oriented concepts, property-oriented concepts, formal fuzzy concepts and three-way concepts \cite{7,8,9,10,11}.

Attribute reduction for formal context plays an essential part in FCA. By attribute reduction, more compact knowledge can be discovered
and the computational complexity for constructing concept lattices can be reduced.
In general, an attribute reduction is a minimal attributes subset which preserves some specific properties of formal context. There are mainly two problems involved in attribute reduction: the criterion of reduction with semantic interpretation and reduction computing method.
For formal context, there are two typical criteria of attribute reduction: (1) To preserve the extents set of all formal concepts calculated from the formal context \cite{12,13,14}. In this case, the concept lattice induced from the reduced context and the one derived from the initial context are isomorphic. (2) To preserve the extents set of all object concepts. This kind of attribute reduction is also called granular reduction \cite{15}.
In order to compute attribute reductions, CR (clarification and reduction) method \cite{12,13} and DM (discernibility matrix) method \cite{14} were proposed. CR method is established by using meet-irreducible elements in formal concept lattice, whereas DM method is based on discernibility attributes between related formal concepts. These two reduction methods have been extensively studied and applied to attribute reductions for various kinds of concept lattices \cite{16,17,18,19,20,21,22,23,24}.

A formal decision context(Fdc) is a formal context in which the attributes are consisted of conditional attributes and decision attributes \cite{25,26}. The knowledge associated with a formal decision context is usually expressed as decision rules to revealing the dependency between conditional and decision attributes. A decision rule is an implication in which the premise and conclusion are concepts of conditional context and decision context respectively. The criteria of attribute reduction for formal decision contexts
can be roughly categorized into two groups: to preserve a kind of consistency \cite{15,27,28,29}, and to preserve a specific kind of decision rules \cite{30,31,32,33,34} of Fdcs.
Qin et al \cite{35} made a comparative research on attribute reduction of formal context and Fdc under the framework of local reduction.

We note that the existing approaches on rule acquisition and attribute reduction pay more attention to decision rules generated by
formal concepts and few work has been completed on other types for decision rules.
Theoretically speaking, decision rules can be designed by using formal concepts, object-oriented concepts or property-oriented concepts.
A specific type of decision rules provide a particular kind of decision knowledge.
In this study, we further investigate attribute reduction methods and rule acquisition methods for Fdcs based on
two kinds of decision rules, namely I-decision rule and II-decision rule.
This paper is structured as follows. In Section 2, the basic definitions of FCA such as formal
concept, property-oriented concept and object-oriented concept are concisely recalled.
In Section 3, we propose algorithms for I-decision rule acquisition and make some comparative analysis with the existing algorithms
presented in \cite{33}. In addition, we present attribute reduction method for Fdc to preserve I-decision rules.
In Section 4, the algorithms for II-decision rule acquisition are presented.
We analyze the relationships between I-decision rules and II-decision rules, and accordingly, attribute reduction method to preserve II-decision rules is examined. Section 5 concludes.

\section{Preliminaries}
In this section, some related notions of FCA are introduced to make this paper self-contained. Please refer to
\cite{1,7,8} for details.

\subsection{Formal context and concept lattice}
The input object-attribute relational data are described by a formal context in FCA.

\begin{definition}\cite{1}\label{def1}
A formal context $(U,M,I)$ constitutes by two sets $U$ and $M$, and a binary relation $I\subseteq U\times M$, where $U$ (objects set) and $M$ (attributes set) are both finite nonempty sets. For $x\in U$ and $a\in M$, $(x,a)\in I$ indicates that the
object $x$ possess the attribute $a$.
\end{definition}
In a formal context $\mathfrak{C}=(U,M,I)$, Wille \cite{1} defined two concept forming operators $\uparrow$ and
$\downarrow$ as follows: for $O\subseteq U$, $C\subseteq M$,
\begin{eqnarray}
O^{\uparrow}=\{a\in M| \forall x\in O((x,a)\in I)\}\\
C^{\downarrow}=\{x\in U| \forall a\in C((x,a)\in I)\}
\end{eqnarray}
That is to say, $O^{\uparrow}$ is the maximal attributes set
had by all objects in $O$, and $C^{\downarrow}$ is the maximal set of objects that possess all attributes in $C$.
A formal concept generated by $\mathfrak{C}$ is a pair $(O, C)$ with two sets $O\subseteq U$ and $C\subseteq M$ such that $O^{\uparrow}=C$
and $C^{\downarrow}=O$, where $O$ and $C$ are regarded as the extent and intent of
$(O, C)$ respectively. Denote the family of all formal concepts of $\mathfrak{C}$ by $L(\mathfrak{C})$. $(L(\mathfrak{C}),\leq)$ constitutes a complete lattice \cite{1}, referred to as the concept lattice of $\mathfrak{C}$, where the order relation $\leq$ is given by:
\begin{eqnarray*}
(O_{i}, C_{i})\leq (O_{j}, C_{j})\Leftrightarrow
O_{i}\subseteq O_{j}
(\Leftrightarrow C_{j}\subseteq C_{i})
\end{eqnarray*}
for any $(O_{i}, C_{i}), (O_{j}, C_{j})\in L(\mathfrak{C})$. In addition, the infimum and supremum of $(L(\mathfrak{C}),\leq)$ are defined as follow:
\begin{eqnarray}
\wedge_{q\in Q}(O_{q}, C_{q})=(\cap_{q\in Q}O_{q}, (\cup_{q\in Q}C_{q})^{\downarrow\uparrow})\\
\vee_{q\in Q}(O_{q}, C_{q})=((\cup_{q\in Q}O_{q})^{\uparrow\downarrow}, \cap_{q\in Q}C_{q})
\end{eqnarray}
where $Q$ is an index set and $\{(O_{q},C_{q})|q\in Q\}\subseteq L(\mathfrak{C})$. For $\forall x\in U$ and $ \forall a\in M$, we write $(\{x\})^{\uparrow}$
and $(\{a\})^{\downarrow}$ simply as  $x^{\uparrow}$ and $a^{\downarrow}$ respectively. In addition $\forall O\subseteq U$ and $\forall A\subseteq M$,
$(O^{\uparrow\downarrow},O^{\uparrow})$ and $(C^{\downarrow},C^{\downarrow\uparrow})$ are both formal concepts.
In what follows, $(O^{\uparrow\downarrow},O^{\uparrow})$ and $(C^{\downarrow},C^{\downarrow\uparrow})$ are referred to as the formal concepts generated by $O$ and $C$ respectively.
Customarily, the formal contexts are all assumed to be canonical \cite{14} in the following discussion, i.e., $\forall x\in U$ and $\forall a\in M$ there has $x^{\uparrow}\neq \emptyset$, $x^{\uparrow}\neq M$, $a^{\downarrow}\neq \emptyset$ and $a^{\downarrow}\neq U$.

\subsection{Property (Object) oriented concept lattice}

FCA and rough set theory \cite{36} are two efficaciously and closely connected mathematical tools for dealing with data.
Over the years, much scholars have been trying to contrast and combine these two theories \cite{7,8,9}.
For a formal context $\mathfrak{C}=(U,M,I)$, based on rough approximation operators, Duntsch and Gediga \cite{7} presented a pair of operators $\lozenge:P(U)\rightarrow P(M)$ and $\square:P(M)\rightarrow P(U)$ as below: for any $O\subseteq U$, $C\subseteq M$
\begin{eqnarray}
O^{\lozenge}=\{a\in M| \exists x\in O((x,a)\in I)\}\\
C^{\square}=\{x\in U| \forall a\in M((x,a)\in I\rightarrow a\in C)\}
\end{eqnarray}
These operators are used to construct property-oriented concepts \cite{7}. Similarly, Yao \cite{8,9} considered a pair of
operators $\square:P(U)\rightarrow P(M)$ and $\lozenge:P(M)\rightarrow P(U)$:
\begin{eqnarray}
O^{\square}=\{a\in M|\forall x \in U((x,a)\in I \rightarrow x \in O)\}\\
C^{\lozenge}=\{x\in U; \exists a \in C((x,a)\in I)\}
\end{eqnarray}
where $O\subseteq U$ and $C\subseteq M$.
Modal-style approximate operators and $\uparrow,\downarrow$ are closely related.
Obviously we know $O^{\lozenge}=\{a\in M| a^{\downarrow}\cap O\neq \emptyset\}$,
$O^{\square}=\{a\in M| a^{\downarrow}\subseteq O\}$,
$C^{\lozenge}=\{x\in U| x^{\uparrow}\cap C\neq \emptyset\}$ and
$C^{\square}=\{x\in U| x^{\uparrow}\subseteq C\}$.
In addition, for any $O_{i}, O_{j}, O_{k}\subseteq U$ and $C_{i}, C_{j}, C_{k}\subseteq M$, the following properties hold:
\begin{enumerate}
  \item $O_{j}\subseteq O_{k}\Rightarrow O_{j}^{\lozenge}\subseteq
O_{k}^{\lozenge}, O_{j}^{\square}\subseteq O_{k}^{\square}$;
  \item $C_{j}\subseteq C_{k}\Rightarrow C_{j}^{\lozenge}\subseteq C_{k}^{\lozenge}, C_{j}^{\square}\subseteq C_{k}^{\square}$;
  \item $O_{i}^{\square\lozenge}\subseteq O_{i} \subseteq O_{i}^{\lozenge\square}, C_{i}^{\square \lozenge}\subseteq C_{i}\subseteq C_{i}^{\lozenge\square}$;
  \item $O_{i}^{\lozenge \square \lozenge}=O_{i}^{\lozenge}, O_{i}^{\square\lozenge\square}=O_{i}^{\square}, C_{i}^{\lozenge\square\lozenge}=C_{i}^{\lozenge}, C_{i}^{\square
\lozenge \square}=C_{i}^{\square}$;
  \item $(O_{j}\cup O_{k})^{\lozenge}=O_{j}^{\lozenge}\cup
O_{k}^{\lozenge}, (O_{j}\cap O_{k})^{\square}=O_{j}^{\square}\cap
O_{k}^{\square}$;
  \item $(C_{j}\cup C_{k})^{\lozenge}=C_{j}^{\lozenge}\cup
C_{k}^{\lozenge}, (C_{j}\cap C_{k})^{\square}=C_{j}^{\square}\cap
C_{k}^{\square}$.
\end{enumerate}

We call a pair $(O, C)$ with $O\subseteq U$ and $C\subseteq M$ a property-oriented concept \cite{7} of $\mathfrak{C}$ if $O^{\lozenge}=C$ and
$C^{\square}=O$. Let $L_{P}(\mathfrak{C})=\{(O,C)|O\subseteq U, C\subseteq M, O^{\lozenge}=C, C^{\square}=O\}$ be the family
of all property-oriented concepts of $\mathfrak{C}$.
$(L_{P}(\mathfrak{C}),\leq)$ is a complete lattice \cite{7}, denoted as the property-oriented concept lattice of $\mathfrak{C}$ with the order relation $\leq$ is given by:
\begin{eqnarray*}
(O_{i}, C_{i})\leq (O_{j}, C_{j})\Leftrightarrow
O_{i}\subseteq O_{j}
(\Leftrightarrow C_{i}\subseteq C_{j})
\end{eqnarray*}
for any $(O_{i}, C_{i}), (O_{j}, C_{j})\in L_{P}(\mathfrak{C})$. The infimum and supremum of $(L_{P}(\mathfrak{C}),\leq)$ are defined as follows:
\begin{eqnarray}
\wedge_{q\in Q}(O_{q}, C_{q})=(\cap_{q\in Q}O_{q}, (\cap_{q\in Q}C_{q})^{\square\lozenge})\\
\vee_{q\in Q}(O_{q}, C_{q})=((\cup_{q\in Q}O_{q})^{\lozenge\square}, \cup_{q\in Q}C_{q})
\end{eqnarray}
$\forall O\subseteq U$ and $\forall C\subseteq M$, $(O^{\lozenge\square},O^{\lozenge})$ and $(C^{\square},C^{\square\lozenge})$ are called the property-oriented concepts derived from $O$ and $C$ separately.

Analogously, we call a pair $(O, C)$ with an objects subset $O$ of and a attributes subset $C$ an object-oriented concept \cite{8} of $\mathfrak{C}$ if $O^{\square}=C$ and $C^{\lozenge}=O$. $L_{O}(\mathfrak{C})$ is referred as the family of all
object-oriented concepts.
$(L_{O}(\mathfrak{C}),\leq)$ is a complete lattice, where the order relation is given by
$(Y_{i},D_{i})\leq (Y_{j},D_{j})\Leftrightarrow Y_{i}\subseteq Y_{j}
(\Leftrightarrow D_{i}\subseteq D_{j})$ and is called the object-oriented concept lattice of $\mathfrak{C}$.
The meet and join of $(L_{O}(\mathfrak{C}),\leq)$ are given by \cite{8}:
\begin{eqnarray}
\wedge_{q\in Q}(O_{q}, C_{q})=(\cap_{q\in Q}O_{q})^{\square\lozenge}, \cap_{q\in Q}C_{q})\\
\vee_{q\in Q}(O_{q}, C_{q})=(\cup_{q\in Q}O_{q}, (\cup_{q\in Q}C_{q})^{\lozenge\square})
\end{eqnarray}
Obviously, $\forall O\subseteq U$ and $\forall C\subseteq M$, $(O^{\square\lozenge},O^{\square})$ and $(C^{\lozenge},C^{\lozenge\square})$ are object-oriented concepts. They are said to be the object-oriented concepts derived from $O$ and $C$ separately.

\section{I-decision rules acquisition and related attribute reduction}

A formal decision context (Fdc) is a formal context in which the attributes are consisted of
conditional attributes and decision attributes.

\begin{definition}\cite{25,26}\label{def1}
A Fdc is a quintuple $\mathfrak{C}=(U,M,I,N,J)$ with two formal contexts $(U,M,I)$ and $(U,N,J)$, $M$ and $N$ are regarded as the sets of conditional attributes and decision attributes respectively with $M\cap N=\emptyset$.
\end{definition}

In addition for a Fdc $\mathfrak{C}=(U,M,I,N,J)$, $(U,M,I)$ and $(U,N,J)$ are called conditional context and decision context of $\mathfrak{C}$ and denoted by
$\mathfrak{C}_{M}=(U,M,I)$ and $\mathfrak{C}_{N}=(U,N,J)$ respectively. In order to distinguish, these operators given by
(1), (2), (5), (6), (7) and (8) for $\mathfrak{C}_{M}$ will be rewrited as $\uparrow_{M}$, $\downarrow_{M}$, $\lozenge_{M}$ and $\square_{M}$,
whereas these operators for $\mathfrak{C}_{N}$ will be denoted by $\uparrow_{N}$, $\downarrow_{N}$, $\lozenge_{N}$ and $\square_{N}$ respectively.

For Fdc $\mathfrak{C}$, we are interested in revealing the dependency relationships between conditional and decision attributes.
It is usually expressed as an implication with the form $(O,C)\rightarrow (Z,D)$ and called decision rule,
where $(O,C)$ and $(Z,D)$ are concepts from $\mathfrak{C}_{M}$ and $\mathfrak{C}_{N}$ respectively.
The rule acquisition and attribute reduction methods with respect to several kinds of decision rules have been extensively investigated, for example:

(1) $(O,C)\rightarrow (Z,D)$: $(O,C)\in L(\mathfrak{C}_{M})$, $(Z,D)\in L(\mathfrak{C}_{N})$, $O\subseteq Z$ and $O,C,Z,D$ are non-empty\cite{29,30,31,32};

(2) $(O,C)\rightarrow (Z,D)$: $(O,C)\in L_{O}(\mathfrak{C}_{M})$, $(Z,D)\in L_{O}(\mathfrak{C}_{N})$, $O\subseteq Z$, $O\neq \emptyset$ and $Z\neq U$ \cite{34};

(3) $(O,C)\rightarrow (Z,D)$: $(O,C)\in L_{P}(\mathfrak{C}_{M})$, $(Z,D)\in L_{P}(\mathfrak{C}_{N})$, $O\subseteq Z$, $O\neq \emptyset$ and $Z\neq U$ \cite{34};

(4) $(O,C)\rightarrow (Z,D)$: $(O,C)\in L_{P}(\mathfrak{C}_{M})$, $(Z,D)\in L(\mathfrak{C}_{N})$, $O\subseteq Z$, $O\neq \emptyset$ and $Z\neq U$ \cite{33};

(5) $(O,C)\rightarrow (Z,D)$: $(O,C)\in L_{O}(\mathfrak{C}_{M})$, $(Z,D)\in L(\mathfrak{C}_{N})$, $O\subseteq Z$, $O\neq \emptyset$ and $Z\neq U$ \cite{33}.

These decision rules are mutually different and present various kinds of decision information.
Ren et al. \cite{33} proposed some rule acquisition algorithms for the decision rule (5).
In this part, we further research the rule acquisition and attribute reduction methods for this kind of decision rules.
We propose new rule acquisition methods and make some comparative study on the rule acquisition algorithms presented in \cite{33} and the
rule acquisition algorithms presented in this paper.
Furthermore, we present related attribute reduction methods.

\subsection{Rule acquisition methods for I-decision rules}

In this subsection, we assume that $\mathfrak{C}=(U,M,I,N,J)$ is a Fdc,
$\mathfrak{C}_{M}=(U,M,I)$ and $\mathfrak{C}_{N}=(U,N,J)$ are the conditional context and decision context of $\mathfrak{C}$ respectively.
The notion of I-decision rules is proposed by Ren et al. \cite{33}. Here we make some modifications on technical terms to fit for this study.

\begin{definition}\cite{33}\label{def1}
Assume that $(O,C)\in L_{O}(\mathfrak{C}_{M})$, $(Y,D)\in
L(\mathfrak{C}_{N})$. If $O\subseteq Y$, $O\neq \emptyset$ and $Y\neq U$, then $(O,C)\rightarrow (Y,D)$ is said to be a
I-decision rule of $\mathfrak{C}$, $(O,C)$ and $(Y,D)$ are the premise and conclusion of $(O,C)\rightarrow (Y,D)$ respectively.
\end{definition}

The semantics of $I$-decision rule $(O,C)\rightarrow (Y,D)$ can be interpreted as follows. By
$C^{\lozenge_{M}}=O\subseteq Y=D^{\downarrow_{N}}$, we know that if an
object $x\in U$ has at least one conditional attribute of $C$, then $x\in C^{\lozenge_{M}}\subseteq D^{\downarrow_{N}}$ and hence $x$ possess all
decision attributes in $D$. In the following, we define
$\mathfrak{R}_{I}(\mathfrak{C})$ as the set of all I-decision rules of $\mathfrak{C}$.

\begin{definition}\label{def1}
Let $(O_{1},C_{1})\rightarrow (Y_{1},D_{1})\in \mathfrak{R}_{I}(\mathfrak{C})$,
$(O_{2},C_{2})\rightarrow (Y_{2},D_{2})\in \mathfrak{R}_{I}(\mathfrak{C})$. If $O_{2}\subseteq
O_{1}\subseteq Y_{1}\subseteq Y_{2}$, then we say that $(O_{2},C_{2})\rightarrow (Y_{2},D_{2})$
can be implied by $(O_{1},C_{1})\rightarrow (Y_{1},D_{1})$ and denoted by $(O_{1},C_{1})\rightarrow (Y_{1},D_{1})\Rightarrow (O_{2},C_{2})\rightarrow (Y_{2},D_{2})$.
\end{definition}

Assume that $(O_{1},C_{1})\rightarrow (Y_{1},D_{1})\Rightarrow (O_{2},C_{2})\rightarrow (Y_{2},D_{2})$.
By $O_{2}\subseteq O_{1}\subseteq Y_{1}\subseteq Y_{2}$, it follows that $C_{2}=O_{2}^{\square_{M}}\subseteq O_{1}^{\square_{M}}=C_{1}$ and $D_{2}=Y_{2}^{\uparrow_{N}}\subseteq Y_{1}^{\uparrow_{N}}=D_{1}$.
If an object $x$ possesses at least one conditional attribute of $C_{2}$, then $x$ possesses at least one conditional attribute of $C_{1}$ and hence
it has all decision attributes in $D_{1}$ by $(O_{1},C_{1})\rightarrow (Y_{1},D_{1})\in \mathfrak{R}_{I}(\mathfrak{C})$.
Consequently, $x$ possess all decision attributes in $D_{2}$ by $D_{2}\subseteq D_{1}$.
We conclude that the decision
information associated with $(O_{2},C_{2})\rightarrow (Y_{2},D_{2})$ can be inferred
from that associated with $(O_{1},C_{1})\rightarrow (Y_{1},D_{1})$.

For $(O,C)\rightarrow (Y,D)\in \mathfrak{R}_{I}(\mathfrak{C})$,
if there exists $(O_{1},C_{1})\rightarrow (Y_{1},D_{1})\in
\mathfrak{R}_{I}(\mathfrak{C})$ such that
$(O_{1},C_{1})\rightarrow (Y_{1},D_{1})\Rightarrow (O,C)\rightarrow
(Y,D)$ and $(O_{1},C_{1})\rightarrow (Y_{1},D_{1})\neq (O,C)\rightarrow
(Y,D)$ (i.e., $(O_{1},C_{1})\neq (O,C)$ or $(Y_{1},D_{1})\neq (Y,D)$), then we call $(O,C)\rightarrow (Y,D)$ is redundant in
$\mathfrak{R}_{I}(\mathfrak{C})$. Otherwise, $(O,C)\rightarrow (Y,D)$ is called
a necessary I-decision rule of $\mathfrak{C}$. Clearly, necessary rules
are more significant than redundant rules. We regard $\mathfrak{\overline{R}}_{I}(\mathfrak{C})$ as the set of all necessary
I-decision rules.

\begin{theorem}\label{thm4}

(1) $(\mathfrak{R}_{I}(\mathfrak{C}),\Rightarrow)$ is a partially ordered set, i.e.,
rule implication relation $\Rightarrow$ satisfies:

a) Reflexivity: $r\Rightarrow r$ for each $r\in \mathfrak{R}_{I}(\mathfrak{C})$;

b) Anti-symmetry: $r_{1}\Rightarrow r_{2}$ and $r_{2}\Rightarrow r_{1}$ imply $r_{1}=r_{2}$ for any $r_{1},r_{2}\in \mathfrak{R}_{I}(\mathfrak{C})$;

c) Transitivity: $r_{1}\Rightarrow r_{2}$ and $r_{2}\Rightarrow r_{3}$ imply $r_{1}\Rightarrow r_{3}$ for any $r_{1},r_{2},r_{3}\in \mathfrak{R}_{I}(\mathfrak{C})$.

(2) $r\in \mathfrak{\overline{R}}_{I}(\mathfrak{C})$ iff $r^{\prime}\Rightarrow r$ implies $r^{\prime}=r$ for any $r^{\prime}\in \mathfrak{R}_{I}(\mathfrak{C})$ in the sense that $r$ is a minimal element of $(\mathfrak{R}_{I}(\mathfrak{C}),\Rightarrow)$.

(3) If $O\in ExtL_{O}(\mathfrak{C}_{M})\cap ExtL(\mathfrak{C}_{N})$, $O\neq \emptyset$ and $O\neq U$, then $(O,O^{\square_{M}})\in L_{O}(\mathfrak{C}_{M})$,
$(O,O^{\uparrow_{N}})\in L(\mathfrak{C}_{N})$ and $(O,O^{\square_{M}})\rightarrow (O,O^{\uparrow_{N}})$
is a necessary I-decision rule.
\end{theorem}
The proof of this Theorem is simple and obvious. We now study the method of necessary I-decision rule acquisition.
Intuitively speaking, a decision rule $(O,C)\rightarrow (Y,D)\in \mathfrak{R}_{I}(\mathfrak{C})$ is necessary if, under the
condition of $O\subseteq Y$, $O$ is as large as possible and $Y$ is as small as possible.
If $(O,C)$ is given, since $O\subseteq O^{\downarrow_{N}\uparrow_{N}}\subseteq Y^{\downarrow_{N}\uparrow_{N}}=Y$, then we have $O^{\downarrow_{N}\uparrow_{N}}$ is the smallest $Y$ such that $O\subseteq Y$ and $Y\in ExtL(\mathfrak{C}_{N})$.
We note that different extents in $ExtL_{O}(\mathfrak{C}_{M})$ may generate same concepts in $L(\mathfrak{C}_{N})$. Therefore,
the object-oriented concepts in $L_{O}(\mathfrak{C}_{M})$ need to be classified.
Let $R_{1}$ be a binary relation on
$ExtL_{O}(\mathfrak{C}_{M})=\{O\subseteq U|\exists C\subseteq M((O,C)\in L_{O}(\mathfrak{C}_{M}))\}$
given by:

\begin{eqnarray}
R_{1}=\{(O,Y)\in ExtL_{O}(\mathfrak{C}_{M})\times ExtL_{O}(\mathfrak{C}_{M})|
O^{\uparrow_{N}}=Y^{\uparrow_{N}}\}
\end{eqnarray}

In other words, $(O,Y)\in R_{1}$ is equivalent to $O$ and $Y$ are extents
in $ExtL_{O}(\mathfrak{C}_{M})$ and they generate same formal concepts in $L(\mathfrak{C}_{N})$.
Clearly, $R_{1}$ is an equivalence relation on $ExtL_{O}(\mathfrak{C}_{M})$.
In the following, we denote by $[O]_{R_{1}}$ the
equivalence class based on $R_{1}$ for $O\in ExtL_{O}(\mathfrak{C}_{M})$. The following theorem
presents an approach to derive necessary I-decision rules.

\begin{theorem}\label{thm4}
For a Fdc $\mathfrak{C}=(U,M,I,N,J)$, we have
\begin{eqnarray}
\mathfrak{\overline{R}}_{I}(\mathfrak{C})=\{(\cup [O]_{R_{1}},
(\cup [O]_{R_{1}})^{\square_{M}})\rightarrow (O^{\uparrow_{N}\downarrow_{N}},O^{\uparrow_{N}})|O\in
ExtL_{O}(\mathfrak{C}_{M}), O\neq \emptyset, O^{\uparrow_{N}\downarrow_{N}}\neq U\}
\end{eqnarray}
\end{theorem}
\begin{proof}
(1) Let $H=\{(\cup [O]_{R_{1}},
(\cup [O]_{R_{1}})^{\square_{M}})\rightarrow (O^{\uparrow_{N}\downarrow_{N}},O^{\uparrow_{N}})|O\in
ExtL_{O}(\mathfrak{C}_{M}), O\neq \emptyset, O^{\uparrow_{N}\downarrow_{N}}\neq U\}$.
We firstly prove that $H\subseteq \mathfrak{\overline{R}}_{I}(\mathfrak{C})$,
i.e.,
$(\cup [O]_{R_{1}},
(\cup [O]_{R_{1}})^{\square_{M}})\rightarrow (O^{\uparrow_{N}\downarrow_{N}},O^{\uparrow_{N}})$ is a necessary I-decision rule for any $O\in ExtL_{O}(\mathfrak{C}_{M})$ with $O\neq \emptyset$ and $O^{\uparrow_{N}\downarrow_{N}}\neq U$.
In fact, For any $Z\in [O]_{R_{1}}$, it follows that $(Z,Z^{\square_{M}})\in L_{O}(\mathfrak{C}_{M})$.
According to formula $(12)$, the supremum of $\{(Z,Z^{\square_{M}})|Z\in [O]_{R_{1}}\}$ in $L_{O}(\mathfrak{C}_{M})$ is given by:
\begin{eqnarray*}
\vee_{Z\in [O]_{R_{1}}}(Z,Z^{\square_{M}})=(\cup [O]_{R_{1}},
(\cup_{Z\in [O]_{R_{1}}} Z^{\square_{M}})^{\lozenge_{M}\square_{M}})
\end{eqnarray*}
Consequently, $\cup [O]_{R_{1}}\in ExtL_{O}(\mathfrak{C}_{M})$. Additionally, $(\cup [O]_{R_{1}})^{\uparrow_{N}}=\cap_{Z\in [O]_{R_{1}}}Z^{\uparrow_{N}}=O^{\uparrow_{N}}$ and therefore $\cup [O]_{R_{1}}\in [O]_{R_{1}}$ is the maximum element in $[O]_{R_{1}}$.
 We have $\cup [O]_{R_{1}}\neq \emptyset$ from $O\neq \emptyset$. By combining that facts
$\cup [O]_{R_{1}}\subseteq (\cup [O]_{R_{1}})^{\uparrow_{N}\downarrow_{N}}=O^{\uparrow_{N}\downarrow_{N}}$, $O^{\uparrow_{N}\downarrow_{N}}\neq U$
and $\cup [O]_{R_{1}}\neq \emptyset$, we conclude
$(\cup [O]_{R_{1}},
(\cup [O]_{R_{1}})^{\square_{M}})\rightarrow (O^{\uparrow_{N}\downarrow_{N}},O^{\uparrow_{N}})$
is indeed a I-decision rule.

Suppose that $(O_{1},C_{1})\rightarrow (Y_{1},D_{1})\Rightarrow
(\cup [O]_{R_{1}},
(\cup [O]_{R_{1}})^{\square_{M}})\rightarrow (O^{\uparrow_{N}\downarrow_{N}},O^{\uparrow_{N}})$ where $(O_{1},C_{1})\rightarrow (Y_{1},D_{1})\in
\mathfrak{R}_{I}(\mathfrak{C})$.
It follows that $(O_{1},C_{1})\in L_{O}(\mathfrak{C}_{M})$, $(Y_{1},D_{1})\in L(\mathfrak{C}_{N})$ and
$\cup [O]_{R_{1}}\subseteq O_{1}\subseteq Y_{1}\subseteq O^{\uparrow_{N}\downarrow_{N}}$.
By $O\subseteq \cup [O]_{R_{1}}\subseteq Y_{1}\subseteq O^{\uparrow_{N}\downarrow_{N}}$ it can be known
$O^{\uparrow_{N}\downarrow_{N}}\subseteq Y_{1}^{\uparrow_{N}\downarrow_{N}}=Y_{1}\subseteq O^{\uparrow_{N}\downarrow_{N}}$.
Consequently $Y_{1}=O^{\uparrow_{N}\downarrow_{N}}$ and hence $(Y_{1},D_{1})=(O^{\uparrow_{N}\downarrow_{N}},O^{\uparrow_{N}})$.
In addition, by $O\subseteq \cup [O]_{R_{1}}\subseteq O_{1}\subseteq O^{\uparrow_{N}\downarrow_{N}}$ we have
$O^{\uparrow_{N}\downarrow_{N}}\subseteq O_{1}^{\uparrow_{N}\downarrow_{N}}\subseteq O^{\uparrow_{N}\downarrow_{N}\uparrow_{N}\downarrow_{N}}=O^{\uparrow_{N}\downarrow_{N}}$, it follows
$O^{\uparrow_{N}\downarrow_{N}}=O_{1}^{\uparrow_{N}\downarrow_{N}}$. Consequently, we know
$O_{1}\in [O]_{R_{1}}$ and $O_{1}\subseteq \cup[O]_{R_{1}}$.
Thus $O_{1}=\cup[O]_{R_{1}}$ and $(O_{1},C_{1})=(\cup [O]_{R_{1}},
(\cup [O]_{R_{1}})^{\square_{N}})$.
We can conclude that $(O_{1},C_{1})\rightarrow (Y_{1},D_{1})=(\cup [O]_{R_{1}},
(\cup [O]_{R_{1}})^{\square_{M}})\rightarrow (O^{\uparrow_{N}\downarrow_{N}},O^{\uparrow_{N}})$
and $(\cup [O]_{R_{1}},
(\cup [O]_{R_{1}})^{\square_{M}})\rightarrow (O^{\uparrow_{N}\downarrow_{N}},O^{\uparrow_{N}})$ is a necessary I-decision rule.

(2) Secondly, we prove that $\mathfrak{\overline{R}}_{I}(\mathfrak{C})\subseteq H$.
Suppose that $(O,C)\rightarrow (Y,D)$ is a necessary I-decision rule.
By (1) we have $(\cup [O]_{R_{1}},
(\cup [O]_{R_{1}})^{\square_{M}})\rightarrow (O^{\uparrow_{N}\downarrow_{N}},O^{\uparrow_{N}})$ is a I-decision rule.
By $O\subseteq Y$ we obtain $O^{\uparrow_{N}\downarrow_{N}}\subseteq Y^{\uparrow_{N}\downarrow_{N}}=Y$ and therefore
$(\cup [O]_{R_{1}},
(\cup [O]_{R_{1}})^{\square_{M}})\rightarrow (O^{\uparrow_{N}\downarrow_{N}},O^{\uparrow_{N}})\Rightarrow (O,C)\rightarrow (Y,D)$.
From $(O,C)\rightarrow (Y,D)$
is necessary, we can conclude $(O,C)\rightarrow (Y,D)=(\cup [O]_{R_{1}},
(\cup [O]_{R_{1}})^{\square_{M}})\rightarrow (O^{\uparrow_{N}\downarrow_{N}},O^{\uparrow_{N}})$.
Consequently $\mathfrak{\overline{R}}_{I}(\mathfrak{C})\subseteq H$ as required.
\end{proof}

In what follows, $(\cup [O]_{R_{1}},
(\cup [O]_{R_{1}})^{\square_{M}})\rightarrow (O^{\uparrow_{N}\downarrow_{N}},O^{\uparrow_{N}})$ is called the necessary I-decision rule generated by $O$. By using Theorem 2 we propose Algorithm 1 to compute necessary I-decision rules.

\begin{algorithm}
\caption{Acquisition of necessary I-decision rules}
\begin{algorithmic}

\State \textbf{Input}: A Fdc $\mathfrak{C}=(U,M,I,N,J)$.

\State \textbf{Output}: $\mathfrak{\overline{R}}_{I}(\mathfrak{C})$// the set of necessary I-decision rules.


\State 1) Construct the object-oriented concept lattices $L_{O}(\mathfrak{C}_{M})$

\State 2) Compute $R_{1}$ by using formula (13)

\State 3) Compute equivalence class $[O]_{R_{1}}$ for each $O\in ExtL_{O}(\mathfrak{C}_{M})$

\State 4) Compute $\mathfrak{\overline{R}}_{I}(\mathfrak{C})$ via applying Theorem 3.4.

\State 5) Output $\mathfrak{\overline{R}}_{I}(\mathfrak{C})$

\end{algorithmic}
\end{algorithm}

Theorem 3.5 ensures the validity of Algorithm 1. Then we analyze its time complexity. If
$L_{O}(\mathfrak{C}_{M})$ is constructed via the algorithms proposed by Outrata and Vychodil \cite{37},
then the running time of Step 1 to construct $L_{O}(\mathfrak{C}_{M})$ is $O(|U||M|^{2}|L_{O}(\mathfrak{C}_{M})|)$.
Running Steps 2-5 takes
$O(|L(\mathfrak{C}_{N})|(|U||M|+|L_{O}(\mathfrak{C}_{M})|))$ in a worst-case. To summary, the global running time is at most
$O(|U||M||L(\mathfrak{C}_{N})|+|L_{O}(\mathfrak{C}_{M})|(|U||M|^{2}+|L(\mathfrak{C}_{N})|))$.

\begin{example}\label{exmp1}
Let us consider a Fdc
$\mathfrak{C}=(U,M,I,N,J)$ presented by Table 1, where $U=\{1,2,3,4,5\}$ is objects set, $M=\{a,b,c,d,e,f\}$ is conditional attributes set,
and $N=\{d_{1},d_{2},d_{3}\}$ is decision attributes set.
The value in Table 1 is $\times$ represents the homologous object possesses the homologous attribute, while not have otherwise.
\begin{table}[htbp]
\centering
\caption{A Fdc} {\begin{tabular}{lccccccccc}
  \hline
&$a$ & $b$ & $c$& $d$ & $e$& $f$& $d_{1}$& $d_{2}$& $d_{3}$\\
  \hline
   $1$ & $\times$ &  &  &  &  &  & $\times$ &  & \\
   $2$ &  & $\times$ &  & $\times$ &  &  & $\times$ & $\times$ & \\
   $3$ & $\times$ &  & $\times$ &  & $\times$ &  & $\times$ & $\times$ & \\
   $4$ &  & $\times$ &  & $\times$ &  & $\times$ &  & $\times$ & $\times$\\
   $5$ & $\times$ & $\times$ & $\times$ &  &  &  & $\times$ & $\times$ & \\
  \hline
\end{tabular}}
\end{table}

\begin{figure}
\centering
\includegraphics[width=2.2in]{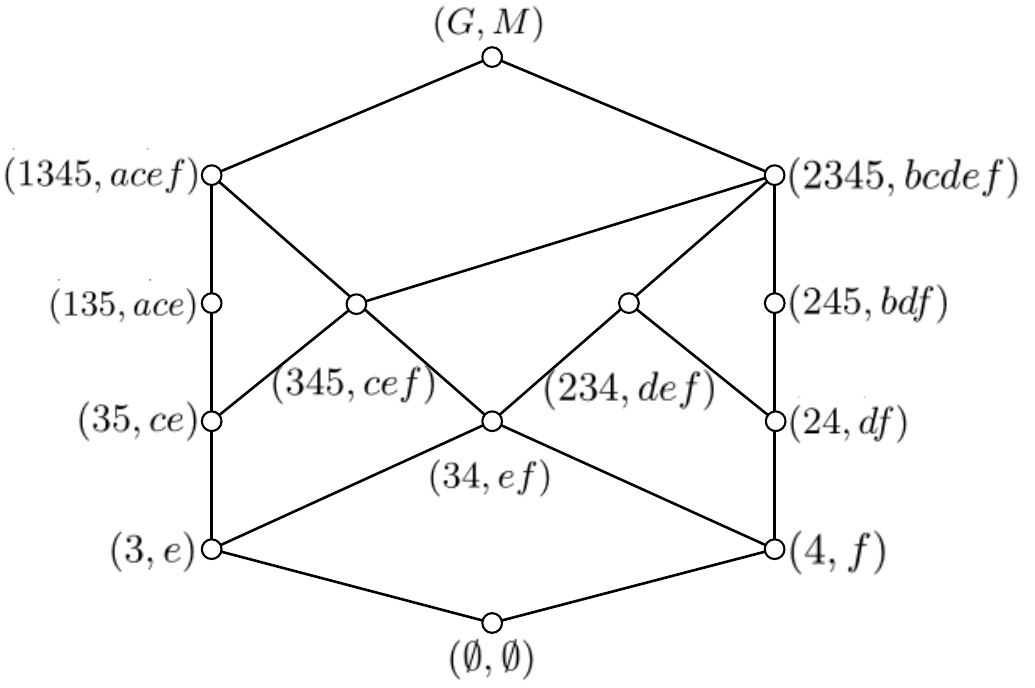}
\caption{Object Oriented  Concept lattice $L_{O}(U, M, I)$}
\small
\label{fig1}
\end{figure}

By direct computation, we have
\begin{eqnarray*}
L_{O}(\mathfrak{C}_{M})=\{(\emptyset, \emptyset), (3,e), (4,f), (24,df),
(34,ef), (35,ce), (135,ace),\\
(234,def),(245,bdf), (345,cef), (1345,acef), (2345,bcdef), (U,M)\}.
\end{eqnarray*}
\begin{eqnarray*}
L(\mathfrak{C}_{N})=\{(\emptyset, N), (4,d_{2}d_{3}), (235,d_{1}d_{2}), (1235,d_{1}),
(2345,d_{2}), (U,\emptyset)\}.
\end{eqnarray*}
Here, for simplicity,
set notion is separator-free, e.g., $245$ substitutes for set $\{2,4,5\}$ and $bdf$ stands for set $\{b,d,f\}$.

The Hasse diagrams of $L_{O}(\mathfrak{C}_{M})$ and $L(\mathfrak{C}_{N})$ are depicted in Fig.1 and Fig.2 respectively.
Additionally, $[\emptyset]_{R_{1}}=\{\emptyset\}$, $[3]_{R_{1}}=\{3,35\}$, $[4]_{R_{1}}=\{4\}$,
$[24]_{R_{1}}=\{24,34,234,245,345,2345\}$, $[135]_{R_{1}}=\{135\}$, $[1345]_{R_{1}}=\{1345, U\}$.
Since $\emptyset^{\uparrow_{N}}=N$, $35^{\uparrow_{N}}=d_{1}d_{2}$, $4^{\uparrow_{N}}=d_{2}d_{3}$,
$2345^{\uparrow_{N}}=d_{2}$, $135^{\uparrow_{N}}=d_{1}$ and $U^{\uparrow_{N}}=\emptyset$,
we have four necessary I-decision rules:

$(r_{1}): (4,f)\rightarrow (4,d_{2}d_{3})$

$(r_{2}): (35,ce)\rightarrow (235,d_{1}d_{2})$

$(r_{3}): (135,ace)\rightarrow (1235,d_{1})$

$(r_{4}): (2345,bcdef)\rightarrow (2345,d_{2})$

We observe that there are fifteen I-decision rules in
$\mathfrak{R}_{I}(\mathfrak{C})$:

$(3,e)\rightarrow (235,d_{1}d_{2}), (3,e)\rightarrow (1235,d_{1}),
(3,e)\rightarrow (2345,d_{1}), (4,f)\rightarrow (4,d_{2}d_{3}),$

$(4,f)\rightarrow (2345,d_{2}), (34,ef)\rightarrow (2345,d_{2}), (24,df)\rightarrow (2345,d_{2}),(35,ce)\rightarrow (235,d_{1}d_{2}),$

$(35,ce)\rightarrow (1235,d_{1}), (35,ce)\rightarrow (2345,d_{2}),(135,ace)\rightarrow (1235,d_{1}), (345,cef)\rightarrow
(2345,d_{2}),$

$(234,def)\rightarrow (2345,d_{2}),(245,bdf)\rightarrow (2345,d_{2}), (2345,bcdef)\rightarrow (2345,d_{2})$.
\begin{figure}
\centering
\includegraphics[width=2.2in]{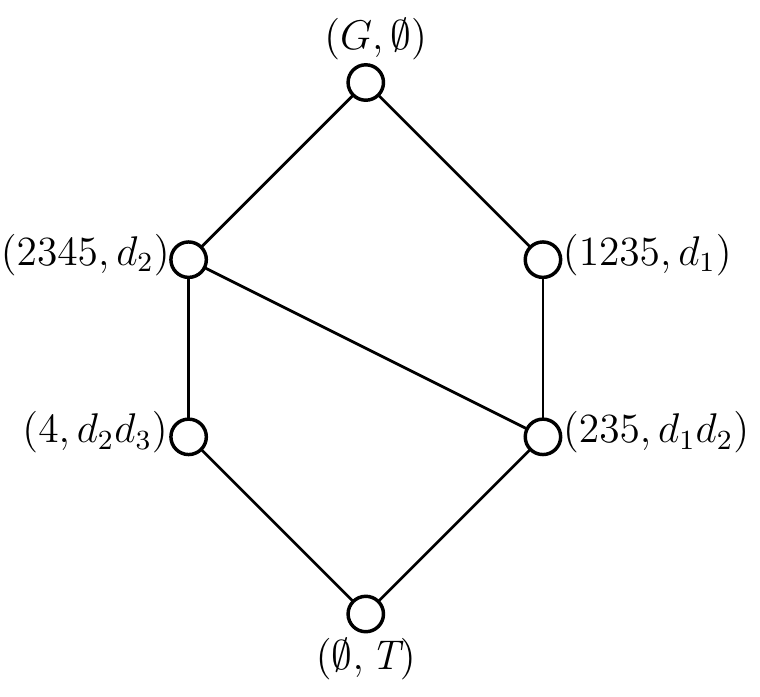}
\caption{Concept lattice $L(\mathfrak{C}_{N})$}
\small
\label{fig2}
\end{figure}
\end{example}

Theorem 3.5 presents an approach to compute necessary I-decision rules via an equivalence relation on $ExtL_{O}(\mathfrak{C}_{M})$.
Actually, decision rules can also be derived based on a classification of formal concepts in $L(\mathfrak{C}_{N})$.
Let $R_{2}$ be a binary relation on $ExtL(\mathfrak{C}_{N})$ given by:
\begin{eqnarray}
R_{2}=\{(O,Y)\in ExtL(\mathfrak{C}_{N})\times ExtL(\mathfrak{C}_{N})|
O^{\square_{M}}=Y^{\square_{M}}\}
\end{eqnarray}
In other words, $(O,Y)\in R_{2}$ equivalent to $O$ and $Y$ are extents in $ExtL(\mathfrak{C}_{N})$ and they generate
same object-oriented concepts in $L_{O}(\mathfrak{C}_{M})$. $R_{2}$ is clearly an equivalence relation on $ExtL(\mathfrak{C}_{N})$.
For each $O\in ExtL(\mathfrak{C}_{N})$ we denote $[O]_{R_{2}}$ as the
equivalence class based on $R_{2}$.
\begin{theorem}\label{thm4}
For a Fdc $\mathfrak{C}=(U,M,I,N,J)$, we have
\begin{eqnarray}
\mathfrak{\overline{R}}_{I}(\mathfrak{C})=\{(O^{\square_{M}\lozenge_{M}}, O^{\square_{M}})
\rightarrow (\cap [O]_{R_{2}}, (\cap [O]_{R_{2}})^{\uparrow_{N}})| O\in
ExtL(\mathfrak{C}_{N}), O\neq U, O^{\square_{M}\lozenge_{M}}\neq \emptyset\}
\end{eqnarray}
\end{theorem}

\begin{proof}
(1) Firstly, we prove that $\{t_{O}| O\in
ExtL(\mathfrak{C}_{N})\}\subseteq \mathfrak{\overline{R}}_{I}(\mathfrak{C})$,
i.e., $t_{O}$ is a necessary I-decision rule for any
$O\in ExtL(\mathfrak{C}_{N})$, where $t_{O}=(O^{\square_{M}\lozenge_{M}}, O^{\square_{M}})
\rightarrow (\cap [O]_{R_{2}}, (\cap [O]_{R_{2}})^{\uparrow_{N}})$.
In fact, For any $Y\in [O]_{R_{2}}$, it follows that $(Y,Y^{\uparrow_{N}})\in L(\mathfrak{C}_{N})$.
By formula $(3)$, the infimum of $\{(Y,Y^{\uparrow_{N}})|Y\in [O]_{R_{2}}\}$ in $L(\mathfrak{C}_{N})$ is given by:
\begin{eqnarray*}
\wedge_{Y\in [O]_{R_{2}}}(Y,Y^{\uparrow_{N}})=(\cap [O]_{R_{2}},
(\cup_{Y\in [O]_{R_{2}}} Y^{\uparrow_{N}})^{\downarrow_{N}\uparrow_{N}})
\end{eqnarray*}
Consequently, $\cap [O]_{R_{2}}\in ExtL(\mathfrak{C}_{N})$. Additionally, by $(\cap [O]_{R_{2}})^{\square_{M}}=\cap_{Y\in [O]_{R_{2}}}Y^{\square_{M}}=O^{\square_{M}}$, we have $\cap [O]_{R_{2}}\in [O]_{R_{2}}$
and $\cap [O]_{R_{2}}$ is clearly the least element in $[O]_{R_{2}}$.
By $O^{\square_{M}\lozenge_{M}}=(\cap [O]_{R_{2}})^{\square_{M}\lozenge_{M}}\\ \subseteq \cap [O]_{R_{2}}$, it
follows that $t_{O}$ is a I-decision rule.

Suppose that $(O_{1},C_{1})\rightarrow (Y_{1},D_{1})\in
\mathfrak{R}_{I}(\mathfrak{C})$ and $(O_{1},C_{1})\rightarrow (Y_{1},D_{1})\Rightarrow t_{O}$.
Then we know that
$O^{\square_{M}\lozenge_{M}}\subseteq O_{1}\subseteq Y_{1}\subseteq \cap [O]_{R_{2}}$.
By $O^{\square_{M}\lozenge_{M}}\subseteq O_{1}\subseteq \cap [O]_{R_{2}}\subseteq O$ we have
$O^{\square_{M}\lozenge_{M}}\subseteq O_{1}=O_{1}^{\square_{M}\lozenge_{M}}\subseteq O^{\square_{M}\lozenge_{M}}$.
Consequently $O_{1}=O^{\square_{M}\lozenge_{M}}$ and hence $(O_{1},C_{1})=(O^{\square_{M}\lozenge_{M}}, O^{\square_{M}})$.
In addition, by $O^{\square_{M}\lozenge_{M}}\subseteq Y_{1}\subseteq \cap [O]_{R_{2}}\subseteq O$ we have
$O^{\square_{M}\lozenge_{M}}\subseteq Y_{1}=Y_{1}^{\square_{M}\lozenge_{M}}\subseteq O^{\square_{M}\lozenge_{M}}$
and thus $Y_{1}^{\square_{M}\lozenge_{M}}=O^{\square_{M}\lozenge_{M}}$. Therefore we obtain
$Y_{1}\in [O]_{R_{2}}$ and hence $\cap [O]_{R_{2}}\subseteq Y_{1}$.
We conclude that $\cap [O]_{R_{2}}=Y_{1}$ and
$(Y_{1},D_{1})=(\cap [O]_{R_{2}}, (\cap [O]_{R_{2}})^{\uparrow_{N}})$. Consequently, $t_{O}=(O_{1},C_{1})\rightarrow (Y_{1},D_{1})$
and $t_{O}$ is thus a necessary I-decision rule.

(2) Secondly, we prove that $\mathfrak{\overline{R}}_{I}(\mathfrak{C})\subseteq \{t_{O}| O\in
ExtL(\mathfrak{C}_{N})\}$.
Assume $(O,C)\rightarrow (Y,D)$ is a necessary I-decision rule. According to (1),
$(Y^{\square_{M}\lozenge_{M}}, Y^{\square_{M}})
\rightarrow (\cap [Y]_{R_{2}}, (\cap [Y]_{R_{2}})^{\uparrow_{N}})$ is a I-decision rule.
By $O\subseteq Y$ and $O\in ExtL_{O}(\mathfrak{C}_{M})$ we have $O=O^{\square_{M}\lozenge_{M}}\subseteq Y^{\square_{M}\lozenge_{M}}$
and hence $O\subseteq Y^{\square_{M}\lozenge_{M}}\subseteq \cap [Y]_{R_{2}}\subseteq Y$.
Consequently, $(Y^{\square_{M}\lozenge_{M}}, Y^{\square_{M}})
\rightarrow (\cap [Y]_{R_{2}}, (\cap [Y]_{R_{2}})^{\uparrow_{N}})\Rightarrow (O,C)\rightarrow (Y,D)$.
By $(O,C)\rightarrow (Y,D)$ is necessary, we obtain $(O,C)\rightarrow (Y,D)=(Y^{\square_{M}\lozenge_{M}}, Y^{\square_{M}})
\rightarrow (\cap [Y]_{R_{2}}, (\cap [Y]_{R_{2}})^{\uparrow_{N}})$.
Thus $((O,C)\rightarrow (Y,D))\in \{t_{O}| O\in
ExtL(\mathfrak{C}_{N})\}$. Consequently, $\mathfrak{\overline{R}}_{I}(\mathfrak{C})\subseteq \{t_{O}| O\in
ExtL(\mathfrak{C}_{N})\}$ as required.
\end{proof}

By using Theorem 3.7, we put forward Algorithm 2 to acquire necessary I-decision rules.

\begin{algorithm}
\caption{Acquisition of necessary I-decision rules}
\begin{algorithmic}

\State \textbf{Input}: A Fdc $\mathfrak{C}=(U,M,I,N,J)$.

\State \textbf{Output}: $\mathfrak{\overline{R}}_{I}(\mathfrak{C})$// the set of necessary I-decision rules.


\State 1) Construct the concept lattices $L(\mathfrak{C}_{N})$

\State 2) Compute $R_{2}$ by using formula (15)

\State 3) Compute equivalence class $[O]_{R_{2}}$ for each $O\in ExtL(\mathfrak{C}_{N})$

\State 4) Using Theorem 3.5, compute $\mathfrak{\overline{R}}_{I}(\mathfrak{C})$

\State 5) Output $\mathfrak{\overline{R}}_{I}(\mathfrak{C})$

\end{algorithmic}
\end{algorithm}

Then we analyze its time complexity. Assume that
$L(\mathfrak{C}_{N})$ is computed by using the algorithms presented in \cite{37}.
Running Step 1 for generating $L(\mathfrak{C}_{N})$ needs $O(|U||N|^{2}|L(\mathfrak{C}_{N})|)$.
The running time of Steps 2-5 is at most $O(|L_{O}(\mathfrak{C}_{M})|(|U||N|+|L(\mathfrak{C}_{N})|))$. Therefore, Algorithm 2 needs at most
$O(|U||N||L_{O}(\mathfrak{C}_{M})|+|L(\mathfrak{C}_{N})|(|U||N|^{2}+|L_{O}(\mathfrak{C}_{M})|))$.

\begin{example}\label{exmp1}
We reconsider the Fdc
$\mathfrak{C}=(U,M,I,N,J)$ given by Table 1. It follows that $[\emptyset]_{R_{2}}=\{\emptyset\}$,
$[4]_{R_{2}}=\{4\}$, $[235]_{R_{2}}=\{235\}$,
$[1235]_{R_{2}}=\{1235\}$, $[2345]_{R_{2}}=\{2345\}$,
$[U]_{R_{2}}=\{U\}$ and $\emptyset^{\square_{M}\lozenge_{M}}=\emptyset$, $4^{\square_{M}\lozenge_{M}}=4$, $235^{\square_{M}\lozenge_{M}}=35$,
$1235^{\square_{M}\lozenge_{M}}=135$, $2345^{\square_{M}\lozenge_{M}}=2345$ and $U^{\square_{M}\lozenge_{M}}=U$.
Hence, we have six necessary I-decision rules:

$(r_{1}): (\emptyset,\emptyset)\rightarrow (\emptyset,N)$

$(r_{2}): (4,f)\rightarrow (4,d_{2}d_{3})$

$(r_{3}): (35,ce)\rightarrow (235,d_{1}d_{2})$

$(r_{4}): (135,ace)\rightarrow (1235,d_{1})$

$(r_{5}): (2345,bcdef)\rightarrow (2345,d_{2})$

$(r_{6}): (U,M)\rightarrow (U,\emptyset)$\\
They are the same as the necessary I-decision rules derived in Example 3.6.
\end{example}

Algorithm 1 and Algorithm 2 are all designed for computing necessary I-decision rules. Algorithm 1 is constructed by
using an equivalence relation $R_{1}$ generated by $ExtL_{O}(\mathfrak{C}_{M})$ and Algorithm 2 is based on an equivalence relation
$R_{2}$ on $ExtL(\mathfrak{C}_{N})$. In general, if $|N|\leq |M|$ and $|ExtL(\mathfrak{C}_{N})|\leq |ExtL_{O}(\mathfrak{C}_{M})|$, then Algorithm 2 is more effective than Algorithm 1. Otherwise, we tend to acquire necessary I-decision rules by using Algorithm 1.

Ren et al. \cite{33} proposed an algorithm to acquire necessary I-decision rules. The time complexity of the
algorithm is
\begin{eqnarray*}
O((|U|+|M|)|M||L_{O}(\mathfrak{C}_{M})|+(|U|+|N|)|N||L(\mathfrak{C}_{N})|\\
+|U||L_{O}(\mathfrak{C}_{M})||L(\mathfrak{C}_{N})|(|L_{O}(\mathfrak{C}_{M})|+|L(\mathfrak{C}_{N})|))
\end{eqnarray*}
Clearly, Algorithm 1 and Algorithm 2 presented in this subsection have lower time complexity than the algorithm presented in \cite{33}.
If necessary I-decision rules are computed by using Algorithm 1, we just need to compute $(O^{\uparrow_{N}\downarrow_{N}},O^{\uparrow_{N}})$ in $L(\mathfrak{C}_{N})$ for $O\in ExtL_{O}(\mathfrak{C}_{M})$.Then we need not to compute the whole concept lattice $L(\mathfrak{C}_{N})$. Similarly, If necessary I-decision rules are computed by using Algorithm 2, $L_{O}(\mathfrak{C}_{M})$ need not to be computed. However, if necessary I-decision rules are computed by using algorithm
presented in \cite{33}, $L_{O}(\mathfrak{C}_{M})$ and $L(\mathfrak{C}_{N})$ are all needed to be computed.

\subsection{Attribute reduction based on I-decision rules}

In this subsection, we present an attribute reduction method for Fdc which preserve I-decision rules.

Let $\mathfrak{C}=(U,M,I,N,J)$ be a Fdc, $\mathfrak{C}_{M}=(U,M,I)$, $\mathfrak{C}_{N}=(U,N,J)$ and
$E\subseteq M$. A Fdc $\mathfrak{C}(E)=(U,E,I_{E},N,J)$ generates by $E$ and $I_{E}=I\cap (U\times E)$,
called a subcontext of $\mathfrak{C}$.
In order to distinguish, the operators given by
(7) and (8) for $(U,E,I_{E})$ will be expressed as $\square_{E}$ and $\lozenge_{E}$ respectively. In other words, $\forall O\subseteq U$, $\forall C\subseteq E$,
we know $O^{\square_{E}}=\{m\in E|\forall x\in U((x,m)\in I\rightarrow x\in O)\}$,
$C^{\lozenge_{E}}=\{g\in U|\exists m\in C((g,m)\in I)\}$.
Obviously it follows $O^{\square_{E}}=O^{\square_{M}}\cap E$ and $C^{\lozenge_{E}}=C^{\lozenge_{M}}$.

\begin{definition}\label{def1}
Assume that $\mathfrak{C}=(U,M,I,N,J)$ is a Fdc, $E\subseteq M$, $\mathfrak{C}(E)=(U,E,I_{E},N,J)$ is the formal decision subcontext
of $\mathfrak{C}$,
$(O,C)\rightarrow (Y,D)\in \mathfrak{R}_{I}(\mathfrak{C}(E))$,
$(O_{1},C_{1})\rightarrow (Y_{1},D_{1})\in \mathfrak{R}_{I}(\mathfrak{C})$. If $O_{1}\subseteq
O\subseteq Y\subseteq Y_{1}$, we call $(O,C)\rightarrow (Y,D)$ imply $(O_{1},C_{1})\rightarrow (Y_{1},D_{1})$, denoted by $(O,C)\rightarrow (Y,D)\Rightarrow
(O_{1},C_{1})\rightarrow (Y_{1},D_{1})$.

If for any $(O_{1},C_{1})\rightarrow (Y_{1},D_{1})\in \mathfrak{R}_{I}(\mathfrak{C})$, there
exists $(O,C)\rightarrow (Y,D)\in \mathfrak{R}_{I}(\mathfrak{C}(E))$ such that
$(O,C)\rightarrow (Y,D)\Rightarrow (O_{1},C_{1})\rightarrow (Y_{1},D_{1})$, then we
say that $\mathfrak{R}_{I}(\mathfrak{C}(E))$ can imply $\mathfrak{R}_{I}(\mathfrak{C})$, denoted by
$\mathfrak{R}_{I}(\mathfrak{C}(E))\Rightarrow \mathfrak{R}_{I}(\mathfrak{C})$.
\end{definition}

\begin{definition}\label{def1}
Let $E\subseteq M$. We call $E$ a I-consistent set of $\mathfrak{C}$ if
$\mathfrak{R}_{I}(\mathfrak{C}(E))\Rightarrow \mathfrak{R}_{I}(\mathfrak{C})$.
In addition, if $E$ is a I-consistent set and
$\forall H\subset E$ is not a I-consistent set of $\mathfrak{C}$, then E is
regarded as a I-reduction of $\mathfrak{C}$.
\end{definition}

From the definition, a I-reduction of $\mathfrak{C}$ is a minimal subset $E$ of conditional attributes such
that the I-decision rules obtained from $\mathfrak{C}$ can be implied by that of $\mathfrak{R}_{I}(\mathfrak{C}(E))$.
In this case, the decision information associated with $\mathfrak{C}$ can be deduced from that of $\mathfrak{C}(E)$.

\begin{theorem}\label{thm5}
Let $E\subseteq M$.
$E$ is a I-consistent set of $\mathfrak{C}$ iff for any $(Y,D)\in
L(\mathfrak{C}_{N})$ and $(O,C)\in L_{O}(\mathfrak{C}_{M})$ with $O\subseteq Y$, there exists $(O',C')\in L_{O}(U,E,I_{E})$ such that $O\subseteq O'\subseteq Y$.
\end{theorem}
\begin{proof}
Suppose $E$ is a I-consistent set of $\mathfrak{C}$, $(Y,D)\in L(\mathfrak{C}_{N})$,
$(O,C)\in L_{O}(\mathfrak{C}_{M})$ and $O\subseteq Y$.
We have
$(O,C)\rightarrow (Y,D)$ is a I-decision rule and hence there
exists $(O_{1},C_{1})\rightarrow (Y_{1},D_{1})\in \mathfrak{R}_{I}(\mathfrak{C}(E))$ such that
$(O_{1},C_{1})\rightarrow (Y_{1},D_{1})\Rightarrow (O,C)\rightarrow (Y,D)$. By Definition 5 we have
$O\subseteq O_{1}\subseteq Y_{1}\subseteq Y$ and consequently
$O\subseteq O_{1}\subseteq Y$ with $(O_{1},C_{1})\in L_{O}(U,E,I_{E})$.

Conversely, assume that $(O,C)\rightarrow (Y,D)\in
\mathfrak{R}_{I}(\mathfrak{C})$. We have $(Y,D)\in L(\mathfrak{C}_{N})$,
$(O,C)\in L_{O}(\mathfrak{C}_{M})$
and $O\subseteq Y$. From the assumption, there exists
$(O',C')\in L_{O}(U,E,I_{E})$ such that $O\subseteq O'\subseteq Y$. From
$O'\subseteq Y$, we know $(O',C')\rightarrow (Y,D)\in
\mathfrak{R}_{I}(\mathfrak{C}(E))$ and $(O',C')\rightarrow (Y,D)\Rightarrow
(O,C)\rightarrow (Y,D)$. We can conclude
$\mathfrak{R}_{I}(\mathfrak{C}(E))\Rightarrow \mathfrak{R}_{I}(\mathfrak{C})$ and $E$ is a I-consistent set of $\mathfrak{C}$.
\end{proof}

\begin{theorem}\label{thm6}
Let $E\subseteq M$. $E$ is a I-consistent set of $\mathfrak{C}$ iff $Y^{\square_{M}\lozenge_{M}}=Y^{\square_{E}\lozenge_{E}}$
for any $Y\in ExtL(\mathfrak{C}_{N})$.
\end{theorem}

\begin{proof}
Suppose $E$ is a I-consistent set of $\mathfrak{C}$ and $Y\in ExtL(\mathfrak{C}_{N})$.
It follows that $Y^{\square_{M}\lozenge_{M}}\in ExtL_{O}(\mathfrak{C}_{M})$ and $Y^{\square_{M}\lozenge_{M}}\subseteq Y$.
From Theorem 3.11,
there exists $Z\in ExtL_{O}(U,E,I_{E})$ such that $Y^{\square_{M}\lozenge_{M}}\subseteq Z\subseteq Y$.
Since
$Y^{\square_{E}\lozenge_{E}}=(Y^{\square_{M}}\cap E)^{\lozenge_{E}}=(Y^{\square_{M}}\cap E)^{\lozenge_{M}}\subseteq Y^{\square_{M}\lozenge_{M}}$,
we obtain
\begin{eqnarray*}
Y^{\square_{M}\lozenge_{M}}\subseteq Z=Z^{\square_{E}\lozenge_{E}}\subseteq Y^{\square_{E}\lozenge_{E}}\subseteq Y^{\square_{M}\lozenge_{M}}
\end{eqnarray*}
and thus $Y^{\square_{M}\lozenge_{M}}=Y^{\square_{E}\lozenge_{E}}$ as required.

Conversely, assume that $(O,C)\in L_{O}(\mathfrak{C}_{M})$, $(Y,D)\in L(\mathfrak{C}_{N})$ and $O\subseteq Y$.
By the assumption and $Y\in ExtL(\mathfrak{C}_{N})$ we obtain $Y^{\square_{M}\lozenge_{M}}=Y^{\square_{E}\lozenge_{E}}$.
Therefore, $O=O^{\square_{M}\lozenge_{M}}\subseteq Y^{\square_{M}\lozenge_{M}}=Y^{\square_{E}\lozenge_{E}}\subseteq Y$.
Thus $O\subseteq Y^{\square_{E}\lozenge_{E}}\subseteq Y$ with $Y^{\square_{E}\lozenge_{E}}\in ExtL_{O}(U,E,I_{E})$. From Theorem 3.11, $E$ is a I-consistent set of $\mathfrak{C}$.
\end{proof}

Let $\mathfrak{U}_{I}(\mathfrak{C})=\{Y^{\square_{M}\lozenge_{M}}| Y\in ExtL(\mathfrak{C}_{N})\}$.
From Theorem 3.7, $O\in \mathfrak{U}_{I}(\mathfrak{C})$ iff $(O,O^{\square_{M}})$ is the premise of a necessary I-decision rule.
For any $(O_{1},C_{1}), (O_{2},C_{2})\in L_{O}(\mathfrak{C}_{M})$, let $\beta
((O_{1},C_{1}), (O_{2},C_{2}))$ be the condition $O_{1}\in
\mathfrak{U}_{I}(\mathfrak{C})\wedge (O_{2},C_{2})\prec (O_{1},C_{1})$ or $O_{2}\in
\mathfrak{U}_{I}(\mathfrak{C})\wedge (O_{1},C_{1})\prec (O_{2},C_{2})$ and
\begin{eqnarray*}
D_{I}((O_{1},C_{1}), (O_{2},C_{2}))=\left\{\begin{array}{ll}C_{1}\cup
C_{2}-C_{1}\cap C_{2}, & \mbox{if}\quad \beta
((O_{1},C_{1}), (O_{2},C_{2})), \\
\emptyset,&\mbox{otherwise}.
\end{array} \right.
\end{eqnarray*}
where $(O_{2},C_{2})\prec (O_{1},C_{1})$ means that $(O_{2},C_{2})$ is a direct sub-concept of $(O_{1},C_{1})$, i.e.,
$(O_{2},C_{2})\leq (O_{1},C_{1})$, $(O_{2},C_{2})\neq (O_{1},C_{1})$ and $(O_{2},C_{2})\leq (O_{3},C_{3}) \leq (O_{1},C_{1})$
implies $(O_{2},C_{2})=(O_{3},C_{3})$ or $(O_{3},C_{3})=(O_{1},C_{1})$.
\begin{theorem}\label{thm8}
Let $E\subseteq M$. $E$ is a I-consistent set of $\mathfrak{C}$ iff for any
$(O_{i},C_{i}), (O_{j},C_{j})\in L_{O}(\mathfrak{C}_{M})$, if $D_{I}((O_{i},C_{i}),
(O_{j},C_{j}))\neq \emptyset$, then $E\cap D_{I}((O_{i},C_{i}),
(O_{j},C_{j}))\neq \emptyset$.
\end{theorem}

\begin{proof}
Necessity. Assume that $(O_{i},C_{i}),
(O_{j},C_{j})\in L_{O}(\mathfrak{C}_{M})$ and $D_{I}((O_{i},C_{i}),
(O_{j},C_{j}))\neq \emptyset$. It follows that $\beta
((O_{i},C_{i}), (O_{j},C_{j}))$ holds. Without losing generality, we suppose that $O_{i}\in
\mathfrak{U}_{I}(\mathfrak{C})\wedge (O_{j},C_{j})\prec (O_{i},C_{i})$.
By $O_{i}\in \mathfrak{U}_{I}(\mathfrak{C})$, there exists $Y\in ExtL(\mathfrak{C}_{N})$ such that
$O_{i}=Y^{\square_{M}\lozenge_{M}}$ and hence $C_{i}=O_{i}^{\square_{M}}=Y^{\square_{M}}$.
Since $E$ is a I-consistent set, we obtain $Y^{\square_{M}\lozenge_{M}}=Y^{\square_{E}\lozenge_{E}}$ from Theorem 3.12. Therefore, from
$(E\cap C_{i})^{\lozenge_{M}}=(E\cap C_{i})^{\lozenge_{E}}=(E\cap Y^{\square_{M}})^{\lozenge_{E}}=Y^{\square_{E}\lozenge_{E}}=Y^{\square_{M}\lozenge_{M}}=O_{i}$,
we have $(E\cap C_{j})^{\lozenge_{M}}\subseteq C_{j}^{\lozenge_{M}}=O_{j}\subset O_{i}=(E\cap C_{i})^{\lozenge_{M}}$ and thus
$E\cap C_{i}\neq E\cap C_{j}$. Therefore $E\cap
(C_{i}\cup C_{j}-C_{i}\cap C_{j})\neq \emptyset$. That is $E\cap
D_{I}((O_{i},C_{i}), (O_{j},C_{j}))\neq \emptyset$.

Sufficiency. By Theorem 3.12, it suffices to prove that $O^{\square_{E}\lozenge_{E}}=O^{\square_{M}\lozenge_{M}}$ for any $O\in ExtL(\mathfrak{C}_{N})$.
If there exists $O\in ExtL(\mathfrak{C}_{N})$ such that $O^{\square_{E}\lozenge_{E}}\neq O^{\square_{M}\lozenge_{M}}$, then $O^{\square_{E}\lozenge_{E}}\subset O^{\square_{M}\lozenge_{M}}$
by $O^{\square_{E}\lozenge_{E}}\subseteq O^{\square_{M}\lozenge_{M}}$.
By combining the facts $(O^{\square_{E}\lozenge_{E}}, O^{\square_{E}\lozenge_{E}\square_{M}})=(O^{\square_{E}\lozenge_{M}}, O^{\square_{E}\lozenge_{M}\square_{M}})\in L_{O}(\mathfrak{C}_{M})$,
$(O^{\square_{M}\lozenge_{M}},\\ O^{\square_{M}})\in L_{O}(\mathfrak{C}_{M})$ and $O^{\square_{E}\lozenge_{E}}\subset O^{\square_{M}\lozenge_{M}}$,
we obtain
$(O^{\square_{E}\lozenge_{E}}, O^{\square_{E}\lozenge_{E}\square_{M}})<(O^{\square_{M}\lozenge_{M}}, O^{\square_{M}})$.
It follows that there exists $(O_{i},C_{i})\in L_{O}(\mathfrak{C}_{M})$ such that $(O^{\square_{E}\lozenge_{E}}, O^{\square_{E}\lozenge_{E}\square_{M}})\leq (O_{i},C_{i})\prec (O^{\square_{M}\lozenge_{M}}, \\ O^{\square_{M}})$.
Consequently, $O^{\square_{E}\lozenge_{E}\square_{M}}\subseteq C_{i}\subset O^{\square_{M}}$.
By the assumption, we obtain $E\cap (O^{\square_{M}}-C_{i})\neq\emptyset$ and therefore $E\cap (O^{\square_{M}}-O^{\square_{E}\lozenge_{E}\square_{M}})\neq \emptyset$.
Hence there exists $e\in E$ such that $e\in O^{\square_{M}}$ and $e\notin O^{\square_{E}\lozenge_{E}\square_{M}}$.
Consequently, $e\in E\cap O^{\square_{M}}=O^{\square_{E}}$. This contradicts the fact that
$O^{\square_{E}}\subseteq O^{\square_{E}\lozenge_{M}\square_{M}}=O^{\square_{E}\lozenge_{E}\square_{M}}$. Thus $O^{\square_{E}\lozenge_{E}}=O^{\square_{M}\lozenge_{M}}$ for each $O\in ExtL_{O}(\mathfrak{C}_{M})$ and $E$ is a I-consistent set.
\end{proof}

For any $(O_{1},C_{1}), (O_{2},C_{2})\in L_{O}(\mathfrak{C}_{M})$,
$D_{I}((O_{1},C_{1}), (O_{2},C_{2}))$ is conditional attributes set which discerns
$(O_{1},C_{1})$ and $(O_{2},C_{2})$. In what follows,
\begin{eqnarray*}
f=\bigwedge_{D_{I}((O_{1},C_{1}), (O_{2},C_{2}))\neq \emptyset}\bigvee
D_{I}((O_{1},C_{1}), (O_{2},C_{2}))
\end{eqnarray*}
is called the discernibility function of $\mathfrak{C}$.
Here each attribute in $D_{I}((O_{1},C_{1}), (O_{2},C_{2}))$ is taken as a Boolean variable
and $f$ is a CNF (conjunctive normal form) formula in classical propositional logic.
By the technique of attribute reduction proposed in \cite{38},
we can get the theorem for computing reductions as follows.

\begin{theorem}\label{thm9}
For a Fdc $\mathfrak{C}=(U,M,I,N,J)$, if the
minimal disjunctive normal form of the discernibility function of
$\mathfrak{C}$ is $f=\bigvee_{i=1}^{t}\bigwedge_{j=1}^{s_{i}} b_{i,j}$,
then $\{E_{i}; 1\leq i\leq t\}$ is the family of all I-reductions of
$\mathfrak{C}$, where $E_{i}=\{b_{i,j}; 1\leq j\leq s_{i}\}$ for any
$1\leq i\leq t$.
\end{theorem}

\begin{example}
We reconsider the Fdc $\mathfrak{C}=(U,M,I,N,J)$
given by Table 1. We can conclude $\mathfrak{U}_{I}(\mathfrak{C})=\{\emptyset,4,35,135,2345,U\}$.
The discernibility matrices are given by Table 2 and
Table 3.
\begin{table}[htbp]
\caption{The discernibility matrix} {\begin{tabular}{lccccccccc}
  \hline
&$(\emptyset,\emptyset)$ & $(3,e)$ & $(4,f)$& $(24,df)$ & $(34,ef)$& $(35,ce)$& $(135,ace)$\\
  \hline
   $(4,f)$ & $f$ &  &  &  &  &  & \\
   $(35,ce)$ &  & $c$ &  &  &  &  & \\
   $(135,ace)$ &  &  &  &  &  & $a$ &  \\
   $(2345,bcdef)$ &  &  &  &  &  &  & \\
   $(U,M)$ &  &  &  &  &  &  & \\
  \hline
\end{tabular}}
\end{table}

\begin{table}[htbp]
\caption{The discernibility matrix} {\begin{tabular}{lccccccccc}
  \hline
& $(234,def)$ & $(245,bdf)$& $(345,cef)$ & $(1345,acef)$& $(2345,bcdef)$& $(U,M)$\\
  \hline
   $(4,f)$ &  &  &  &  &  & \\
   $(35,ce)$ &  &  &  &  &  & \\
   $(135,ace)$ &  &  &  &  &  & \\
   $(2345,bcdef)$ & $bc$ & $ce$ & $bd$ &  &  & \\
   $(U,M)$ &  &  &  & $bd$ & $a$ & \\
  \hline
\end{tabular}}
\end{table}
The discernibility function of $\mathfrak{C}$ is

$f=f\wedge c\wedge a\wedge (b\vee c)\wedge (c\vee e)\wedge (b\vee d)=f\wedge c\wedge a\wedge (b\vee d)$

$=(a\wedge b\wedge c\wedge f)\vee (a\wedge c\wedge d\wedge f)$

Therefore, there are two I-reductions: $\{a,b,c,f\}$ and
$\{a,c,d,f\}$.
\end{example}

\section{II-decision rule acquisition and related attribute reduction}
In what follows, we consider another type of decision rules which is generated by object-oriented concept and property-oriented
concept.
\begin{definition}\label{def1}
In a Fdc $\mathfrak{C}=(U,M,I,N,J)$, for any $(O,C)\in L_{O}(\mathfrak{C}_{M})$ and $(Y,D)\in
L_{P}(\mathfrak{C}_{N})$ with $O\subseteq Y$, $(O,C)\rightarrow (Y,D)$ is said to be a
II-decision rule of $\mathfrak{C}$.
\end{definition}

Assume $(O,C)\rightarrow (Y,D)$ is a $II$-decision rule. By the definition, $(O,C)$ is an object-oriented concept
generated by conditional context $\mathfrak{C}_{M}=(U,M,I)$ while $(Y,D)$ is a property-oriented concept
in decision context $\mathfrak{C}_{N}=(U,T,J)$. From
$C^{\lozenge_{M}}=O\subseteq Y=D^{\square_{N}}$, we can conclude if an
object $x\in U$ has some conditional attributes in $C$, then $x\in C^{\lozenge_{M}}\subseteq D^{\square_{N}}$
and hence the decision attributes had by $x$ are all in $D$.
We denote
$\mathfrak{R}_{II}(\mathfrak{C})$ as the set of all II-decision rules of $\mathfrak{C}$.

\begin{definition}\label{def1}
For $(O_{1},C_{1})\rightarrow (Y_{1},D_{1})\in \mathfrak{R}_{II}(\mathfrak{C})$,
$(O_{2},C_{2})\rightarrow (Y_{2},D_{2})\in \mathfrak{R}_{II}(\mathfrak{C})$, if $O_{2}\subseteq
O_{1}\subseteq Y_{1}\subseteq Y_{2}$, then we say that $(O_{2},C_{2})\rightarrow (Y_{2},D_{2})$
can be implied by $(O_{1},C_{1})\rightarrow (Y_{1},D_{1})$ and denote this
implication relationship by $(O_{1},C_{1})\rightarrow (Y_{1},D_{1})\Rightarrow (O_{2},C_{2})\rightarrow (Y_{2},D_{2})$.
\end{definition}

Assume that $(O_{1},C_{1})\rightarrow (Y_{1},D_{1})\Rightarrow (O_{2},C_{2})\rightarrow (Y_{2},D_{2})$.
By $O_{2}\subseteq O_{1}\subseteq Y_{1}\subseteq Y_{2}$, it follows that $C_{2}=O_{2}^{\square_{M}}\subseteq O_{1}^{\square_{M}}=C_{1}$ and
$D_{1}=Y_{1}^{\lozenge_{N}}\subseteq Y_{2}^{\lozenge_{N}}=D_{2}$.
We conclude that the decision
information associated with $(O_{2},C_{2})\rightarrow (Y_{2},D_{2})$ can be inferred
from $(O_{1},C_{1})\rightarrow (Y_{1},D_{1})$.

Let $(O,C)\rightarrow (Y,D)\in \mathfrak{R}_{II}(\mathfrak{C})$.
If there exists $(O_{1},C_{1})\rightarrow (Y_{1},D_{1})\in
\mathfrak{R}_{II}(\mathfrak{C})-\{(O,C)\rightarrow (Y,D)\}$ such that
$(O_{1},C_{1})\rightarrow (Y_{1},D_{1})\Rightarrow (O,C)\rightarrow
(Y,D)$, then $(O,C)\rightarrow (Y,D)$ is called a redundant rule in
$\mathfrak{R}_{II}(\mathfrak{C})$. Otherwise, $(O,C)\rightarrow (Y,D)$ is referred to as a necessary rule in $\mathfrak{R}_{II}(\mathfrak{C})$. We denote
$\mathfrak{\overline{R}}_{II}(\mathfrak{C})$ as the set of all necessary
II-decision rules.

In order to acquire necessary II-decision rules, we define binary relation $S_{1}$ on
$ExtL_{O}(\mathfrak{C}_{M})$ and $S_{2}$ on $ExtL_{P}(\mathfrak{C}_{N})$. Let
\begin{eqnarray}
S_{1}=\{(O,Y)\in ExtL_{O}(\mathfrak{C}_{M})\times ExtL_{O}(\mathfrak{C}_{M})|
O^{\lozenge_{N}}=Y^{\lozenge_{N}}\}
\end{eqnarray}
\begin{eqnarray}
S_{2}=\{(O,Y)\in ExtL_{P}(\mathfrak{C}_{N})\times ExtL_{P}(\mathfrak{C}_{N})|
O^{\square_{M}}=Y^{\square_{M}}\}
\end{eqnarray}
Clearly, $S_{1}$ and $S_{2}$ are all equivalence relations.
We denote $[O]_{S_{1}}$ as the
equivalence class based on $O$ for each $O\in ExtL_{O}(\mathfrak{C}_{M})$ and
by $[O]_{S_{2}}$ the
equivalence class based on $O$ for each $O\in ExtL_{P}(\mathfrak{C}_{N})$ respectively.
The following theorems
present approaches to derive necessary II-decision rules.

\begin{theorem}\label{thm4}
For a Fdc $\mathfrak{C}=(U,M,I,N,J)$, we have
\begin{eqnarray}
\mathfrak{\overline{R}}_{II}(\mathfrak{C})=\{(\cup [O]_{S_{1}},
(\cup [O]_{S_{1}})^{\square_{M}})\rightarrow (O^{\lozenge_{N}\square_{N}},O^{\lozenge_{N}})| O\in
ExtL_{O}(\mathfrak{C}_{M})\}
\end{eqnarray}
\end{theorem}

\begin{theorem}\label{thm4}
For a Fdc $\mathfrak{C}=(U,M,I,N,J)$, we have
\begin{eqnarray}
\mathfrak{\overline{R}}_{II}(\mathfrak{C})=\{(O^{\square_{M}\lozenge_{M}}, O^{\square_{M}})
\rightarrow (\cap [O]_{S_{2}}, (\cap [O]_{S_{2}})^{\lozenge_{N}})| O\in
ExtL_{P}(\mathfrak{C}_{N})\}
\end{eqnarray}
\end{theorem}

We can prove Theorem 4.3 and Theorem 4.4 similarly to Theorem 3.5 and Theorem 3.7 respectively.
\begin{example}
We reconsider the Fdc $\mathfrak{C}=(U,M,I,N,J)$
given by Table 1. It can be computed that
\begin{eqnarray*}
L_{O}(\mathfrak{C}_{M})=\{(\emptyset, \emptyset), (3,e), (4,f), (24,df),
(34,ef), (35,ce), (135,ace), \\
(234,def),(245,bdf), (345,cef), (1345,acef), (2345,bcdef), (U,M)\}.
\end{eqnarray*}
In addition, we have
\begin{eqnarray*}
L_{P}(\mathfrak{C}_{N})=\{(\emptyset, \emptyset), (1,d_{1}), (4,d_{2}d_{3}),
(1235,d_{1}d_{2}), (U,N)\}.
\end{eqnarray*}

We consider $S_{1}$ on $ExtL_{O}(\mathfrak{C}_{M})$.
It's routine to review that $[\emptyset]_{S_{1}}=\{\emptyset\}$,
$[3]_{S_{1}}=\{3,35,135\}$, $[4]_{S_{1}}=\{4\}$,
$[24]_{S_{1}}=\{24,34,234,345,245,1345, 2345, U\}$
and $\emptyset^{\lozenge_{N}\square_{N}}=\emptyset$, $3^{\lozenge_{N}\square_{N}}=1235$, $4^{\lozenge_{N}\square_{N}}=4$,
$24^{\lozenge_{N}\square_{N}}=U$.
Thus, by Theorem 4.3, there are four necessary II-decision rules:

$r_{1}:(\emptyset,\emptyset)\rightarrow (\emptyset,\emptyset)$

$r_{2}:(4,f)\rightarrow (4,d_{2}d_{3})$

$r_{3}:(135,ace)\rightarrow (1235,d_{1}d_{2})$

$r_{4}:(U,M)\rightarrow (U,N)$.

Now we consider $S_{2}$ on $ExtL_{P}(\mathfrak{C}_{N})$. It is routine to check that $[\emptyset]_{S_{2}}=\{1,\emptyset\}$,
$[4]_{S_{2}}=\{4\}$,
$[1235]_{S_{2}}=\{1235\}$, $[U]_{S_{2}}=\{U\}$
and $\emptyset^{\square_{M}\lozenge_{M}}=\emptyset$, $4^{\square_{M}\lozenge_{M}}=4$, $1235^{\square_{M}\lozenge_{M}}=135$, $U^{\square_{M}\lozenge_{M}}=U$.
Thus, by Theorem 4.4, we also have four necessary II-decision rules $r_{1}$, $r_{2}$, $r_{3}$ and $r_{4}$.
\end{example}

Similar to I-decision rules, in practical application, II-decision rules $(O,C)\rightarrow (Y,D)$ will be restricted by $O\neq\emptyset$ and $Y\neq U$.
In this case, we will obtain necessary II-decision rules $r_{2}$ and $r_{3}$ in Example 4.5.

For a Fdc $\mathfrak{C}=(U,M,I,N,J)$, $\mathfrak{C}^{c}=(U,M,I,N,\neg J)$ is called the complement Fdc
of $\mathfrak{C}$ where $(g,t)\in \neg J$ is determined by $(g,t)\notin J$ for any $g\in U$ and $t\in N$.

\begin{theorem}\cite{8} $\varphi: L(U,N,\neg J)\rightarrow L_{P}(U,N,J)$ is a lattice isomorphism,
where $\varphi ((O,C))=(O,M-C)$ for any $(O,C)\in L(U,N,\neg J)$.
\end{theorem}

From this theorem, $L(U,N,\neg J)$ and $L_{P}(U,N,J)$ are isomorphic. Thus we have the next theorem.

\begin{theorem}
For a Fdc $\mathfrak{C}=(U,M,I,N,J)$, we have

(1) $\mathfrak{R}_{II}(\mathfrak{C})=\{(O,C)\rightarrow (Y,D)|(O,C)\rightarrow (Y,\neg D)\in \mathfrak{R}_{I}^{c}(\mathfrak{C})\}$

(2) $\mathfrak{\overline{R}}_{II}(\mathfrak{C})=\{(O,C)\rightarrow (Y,D)|(O,C)\rightarrow (Y,\neg D)\in \mathfrak{\overline{R}}_{I}^{c}(\mathfrak{C})\}$\\
where $\mathfrak{R}_{I}^{c}(\mathfrak{C})$ is the family of all I-decision rules and $\mathfrak{\overline{R}}_{I}^{c}(\mathfrak{C})$
is the family of all necessary I-decision rules of the complement Fdc $\mathfrak{C}^{c}$.
\end{theorem}

\begin{definition}\label{def1}
Let $E\subseteq M$, $(O,C)\rightarrow (Y,D)\in \mathfrak{R}_{II}(\mathfrak{C}(E))$,
$(O_{1},C_{1})\rightarrow (Y_{1},D_{1})\in \mathfrak{R}_{II}(\mathfrak{C})$. If $O_{1}\subseteq
O\subseteq Y\subseteq Y_{1}$, then we call $(O,C)\rightarrow (Y,D)$ imply $(O_{1},C_{1})\rightarrow (Y_{1},D_{1})$ and denote this
implication relationship by $(O,C)\rightarrow (Y,D)\Rightarrow
(O_{1},C_{1})\rightarrow (Y_{1},D_{1})$.

If for any $(O_{1},C_{1})\rightarrow (Y_{1},D_{1})\in \mathfrak{R}_{II}(\mathfrak{C})$, there
exists $(O,C)\rightarrow (Y,D)\in \mathfrak{R}_{II}(\mathfrak{C}(E))$ such that
$(O,C)\rightarrow (Y,D)\Rightarrow (O_{1},C_{1})\rightarrow (Y_{1},D_{1})$, then we
say that $\mathfrak{R}_{II}(\mathfrak{C}(E))$ imply $\mathfrak{R}_{II}(\mathfrak{C})$, denoted by
$\mathfrak{R}_{II}(\mathfrak{C}(E))\Rightarrow \mathfrak{R}_{II}(\mathfrak{C})$.
\end{definition}

\begin{definition}\label{def1}
In a Fdc $\mathfrak{C}=(U,M,I,N,J)$, $E\subseteq M$, we call $E$ a II-consistent set of $\mathfrak{C}$ if
$\mathfrak{R}_{II}(\mathfrak{C}(E))\\ \Rightarrow \mathfrak{R}_{II}(\mathfrak{C})$.
In addition, if $E$ is a II-consistent set and $\forall H\subset E$ is not a II-consistent set of $\mathfrak{C}$, then we call E is
a II-reduction of $\mathfrak{C}$.
\end{definition}

\begin{theorem}\label{thm11}
Let $\mathfrak{C}=(U,M,I,N,J)$ be a Fdc, $E\subseteq M$.

(1) $E$ is a II-consistent set of $\mathfrak{C}$ iff $E$ is a I-consistent set of $\mathfrak{C}^{c}$.

(2) $E$ is a II-reduction of $\mathfrak{C}$ iff $E$ is a I-reduction of $\mathfrak{C}^{c}$.
\end{theorem}
\begin{proof}
(1) Suppose that $E$ is a II-consistent set of $\mathfrak{C}$. For each $(O,C)\rightarrow (Y,D)\in \mathfrak{R}_{I}^{c}(\mathfrak{C})$,
we have $(O,C)\rightarrow (Y,T-D)\in \mathfrak{R}_{II}(\mathfrak{C})$ and therefore there exists $(O_{1},C_{1})\rightarrow (Y_{1},D_{1})\in \mathfrak{R}_{II}(\mathfrak{C}(E))$ such that $(O_{1},C_{1})\rightarrow (Y_{1},D_{1})\Rightarrow (O,C)\rightarrow (Y,T-D)$ because of $E$ is a II-consistent set of $\mathfrak{C}$. We have $O\subseteq O_{1}\subseteq Y_{1}\subseteq Y$ and
$(O_{1},C_{1})\rightarrow (Y_{1},T-D_{1})\in \mathfrak{R}_{I}^{c}(\mathfrak{C}(E))$.
It follows that $(O_{1},C_{1})\rightarrow (Y_{1},T-D_{1})\Rightarrow (O,C)\rightarrow (Y,D)$ and
$E$ is a I-consistent set of $\mathfrak{C}^{c}$ as required.

If $E$ is a I-consistent set of $\mathfrak{C}^{c}$, then $E$ is a II-consistent set of $\mathfrak{C}$ can be proved similarly.

(2) follows directly from (1).
\end{proof}

\begin{example}
We reconsider the Fdc $\mathfrak{C}$
given by Table 1. The complement Fdc $\mathfrak{C}^{c}$
is proposed by Table 4.

\begin{table}[htbp]
\centering
\caption{Formal decision context $\mathfrak{C}^{c}$} {\begin{tabular}{lccccccccc}
  \hline
&$a$ & $b$ & $c$& $d$ & $e$& $f$& $d_{1}$& $d_{2}$& $d_{3}$\\
  \hline
   $1$ & $\times$ &  &  &  &  &  &  & $\times$ & $\times$\\
   $2$ &  & $\times$ &  & $\times$ &  &  &  &  & $\times$\\
   $3$ & $\times$ &  & $\times$ &  & $\times$ &  &  &  & $\times$\\
   $4$ &  & $\times$ &  & $\times$ &
    & $\times$ & $\times$ &  & \\
   $5$ & $\times$ & $\times$ & $\times$ &  &  &  &  &  & $\times$\\
  \hline
\end{tabular}}
\end{table}
\begin{table}[htbp]
\caption{The discernibility matrix} {\begin{tabular}{lccccccccc}
  \hline
&$(\emptyset,\emptyset)$ & $(3,e)$ & $(4,f)$& $(24,df)$ & $(34,ef)$& $(35,ce)$& $(135,ace)$\\
  \hline
   $(\emptyset,\emptyset)$ & &  &  &  &  &  & \\
   $(4,f)$ & $f$ &  &  &  &  &  & \\
   $(135,ace)$ &  &  &  &  &  & $a$ &  \\
   $(U,M)$ &  &  &  &  &  &  & \\
  \hline
\end{tabular}}
\end{table}

\begin{table}[htbp]
\caption{The discernibility matrix} {\begin{tabular}{lccccccccc}
  \hline
& $(234,def)$ & $(245,bdf)$& $(345,cef)$ & $(1345,acef)$& $(2345,bcdef)$& $(U,M)$\\
  \hline
   $(\emptyset,\emptyset)$ &  &  &  &  &  & \\
   $(4,f)$ &  &  &  &  &  & \\
   $(135,ace)$ &  &  &  &  &  & \\
   $(U,M)$ &  &  &  & $bd$ & $a$ & \\
  \hline
\end{tabular}}
\end{table}
By Example 3.6, we obtain
\begin{eqnarray*}
L_{O}(\mathfrak{C}_{M})=\{(\emptyset, \emptyset), (3,e), (4,f), (24,df),
(34,ef), (35,ce), (135,ace), \\
(234,def),(245,bdf), (345,cef), (1345,acef), (2345,bcdef), (U,M)\}.
\end{eqnarray*}
In addition, we have
\begin{eqnarray*}
L(U,T,\neg J)=\{(\emptyset, T), (1,d_{2}d_{3}), (4,d_{1}),
(1235,d_{3}), (U,\emptyset)\}.
\end{eqnarray*}

Therefore, $\mathfrak{U}_{I}(\mathfrak{C}^{c})=\{\emptyset,4,135,U\}$.
The discernibility matrices are given by Table 5 and
Table 6.

The discernibility function of $\mathfrak{C}^{c}$ is $f\wedge a\wedge (b\vee d)=(a\wedge b\wedge f)\vee(a\wedge d\wedge f)$
and $\mathfrak{C}^{c}$ has two I-reduction
$\{a,b,f\}$ and $\{a,d,f\}$. Then we can obtain $\{a,b,f\}$ and $\{a,d,f\}$ are II-reduction of $\mathfrak{C}$.
\end{example}

\section{Conclusions}
The knowledge of formal decision context is usually expressed as decision rules. We note that the existing works on knowledge discovery
of formal decision context focus mainly on decision rules derived from conditional and decision formal concepts.
This paper mainly proposes some novel methods to knowledge discovery for formal decision context based on two
new kinds of decision rules, namely I-decision rules and II-decision rules, which are generated by formal concepts, object-oriented concepts
and property-oriented concepts.
For I-decision rules, via equivalence relations on extents set of conditional (decision) concept lattices,
we develop two rule acquisition algorithms.
Some comparative analysis of these algorithms with the existing algorithms
presented in \cite{33} is conducted. It is shown that the algorithms presented in this paper have lower time complexities than
the existing ones. The attribute reduction method for formal decision context to preserve I-decision rules is presented.
For II-decision rules,
by using isomorphism between concept lattice of a formal context and property-oriented concept lattice of its complement context,
the algorithms for rule acquisition are proposed and the attribute reduction method to preserve II-decision rules is examined.

In further research, we will study attribute reduction methods for formal decision context with respect to some other types
of decision rules.
Moreover, the applications of attribute reduction approaches to three-way decision theory deserve further study.

\section*{Acknowledgements}
This work was supported by the National Natural
Science Foundation of China (Grant No. 61976130).


\end{document}